\documentclass[11pt]{article}

\usepackage{geometry}
\usepackage{natbib}

\usepackage{pgfplots}

\usepackage{microtype}
\usepackage{graphicx}
\usepackage{booktabs}
\usepackage{array}

\usepackage{hyperref}

\usepackage{color}
\usepackage{booktabs}       
\usepackage{amsfonts}       
\usepackage{nicefrac}       
\usepackage{microtype}      
\usepackage{lipsum}
\usepackage{bbm}
\usepackage{dsfont}

\usepackage{pgffor}
\usepackage{bbm}
\usepackage{graphicx}
\usepackage{subcaption}
\usepackage{multirow}
\usepackage[normalem]{ulem}

\usepackage{amsmath}

\newcommand{\comm}[1]{\textcolor{cyan}{}}

\usepackage{amsthm}

\newtheorem{definition}{Definition}
\newtheorem{proposition}{Proposition}

\newtheorem{lemma}{Lemma}
\newtheorem{corollary}{Corollary}

\theoremstyle{definition}
\newtheorem{example}{Example}[section]

\title{Lifted Inference in 2-Variable Markov Logic Networks with Function and Cardinality Constraints Using DFT}

\author{%
   Ond\v{r}ej Ku\v{z}elka \\
   Faculty of Electrical Engineering \\
   Czech Technical University in Prague \\
   Prague, Czech Republic \\
}

\begin{document}

\maketitle

\begin{abstract}
In this paper we show that inference in 2-variable Markov logic networks (MLNs) with cardinality and function constraints is domain-liftable. To obtain this result we use existing domain-lifted algorithms for weighted first-order model counting (Van den Broeck et al, KR 2014) together with discrete Fourier transform of certain distributions associated to MLNs.
\end{abstract}

\section{Introduction}

Markov logic networks (MLNs, \citeauthor{Richardson2006}, \citeyear{Richardson2006}) are a statistical relational learning \citep{getoor2007introduction} framework for probabilistic modelling of complex relational structures such as social and biological networks, molecules etc. In general, inference in MLNs is intractable. {\em Lifted inference} refers to a set of methods developed in the literature which exploit symmetries for making probabilistic inference more tractable, e.g.\ \citep{DBLP:conf/ijcai/BrazAR05,DBLP:conf/uai/GogateD11a,van2011lifted,van2014skolemization,DBLP:conf/nips/KazemiKBP16}. In particular, there exist restricted classes of MLNs for which inference is polynomial-time. Such MLNs are called {\em domain-liftable} and the most prominent among them are so-called $2$-variable MLNs \citep{van2011lifted,van2014skolemization}.

Recently, \citeauthor{DBLP:conf/lics/KuusistoL18} (\citeyear{DBLP:conf/lics/KuusistoL18}) obtained a result for {\em weighted first-order model counting} which extends the domain-liftability of 2-variable MLNs by allowing to add one {\em function} constraint, allowing to specify that some binary relation should behave as a function, while still guaranteeing polynomial-time inference using a rather involved approach. 
In this paper, we show a simpler way to add an arbitrary number of {\em function} constraints and {\em cardinality} constraints while maintaining polynomial-time inference. We build on our previous work \citep{kuzelka.complex} in which we noticed usefulness of complex weights in MLNs.

\section{Background}

\subsection{Notation}

We use $i$ to denote the imaginary unit $i^2 = -1$. For a vector $\mathbf{v}$, we use $[\mathbf{v}]_j$ to denote its $j$-th component. We use $\langle v, w \rangle$ to denote the inner product of the vectors $v$ and $w$ (when $v$ and $w$ are real vectors, inner product coincides with scalar product).

\subsection{Discrete Fourier Transform}\label{sec:fourier}

Let $d$ be a positive integer and let $\mathbf{N} = [N_1, \dots, N_d] \in (\mathbf{N} \setminus \{ 0 \})^d$ be a vector of positive integers. Let us define $\mathcal{J} = \{0,1,\dots, N_1 - 1 \} \times \{0,1,\dots, N_2 - 1 \} \times \dots \times \{0,1,\dots, N_d - 1 \}$. Let $f : \mathcal{J} \rightarrow \mathbb{C}$ be a function defined on $\mathcal{J}$. Then the {\em discrete Fourier transform} (DFT) of $f$ is the function $g : \mathcal{J} \rightarrow \mathbb{C}$ defined as
\begin{equation}\label{eq:dft}
    g(\mathbf{k}) = \sum_{\mathbf{n} \in \mathcal{J}} f(\mathbf{n})  e^{ - i 2 \pi \langle \mathbf{k}, \mathbf{n} / \mathbf{N} \rangle }
\end{equation}
where $\mathbf{k}/\mathbf{N} \stackrel{def}{=} \left[ [\mathbf{k}]_1/N_1, [\mathbf{k}]_2/N_2, \dots, [\mathbf{k}]_d/N_d \right]$ (i.e.\ ``/'' denotes component-wise division). 

\subsection{First Order Logic}

We assume a function-free first-order logic (FOL) language defined by a set of constants, a set of variables and a set of predicates (relations). When there is no risk of confusion, we assume such a language implicitly and do not specify its components $\mathcal{V}$, $\mathcal{R}$ (although we will usually specify the domain). Variables start with lowercase letters and constants start with uppercase letters. An atom is $r(a_1,...,a_k)$ with $a_1,...,a_k\in \Delta \cup \mathcal{V}$ and $r\in \mathcal{R}$. A literal is an atom or its negation. For an FOL formula $\alpha$, we define $\textit{vars}(\alpha)$ to be the set of variables contained in it which are not bound to any quantifier.
An FOL formula in which none of the literals contains any variables is called {\em ground}. A possible world $\omega$ is represented as a set of ground atoms that are true in $\omega$. The satisfaction relation $\models$ is defined in the usual way: $\omega \models \alpha$ means that the formula $\alpha$ is true in $\omega$.

\subsection{Markov Logic Networks}

A Markov logic network \citep{Richardson2006} (MLN) is a set of weighted first-order logic formulas $(\alpha,w)$, where $w\in \mathbb{R}$ and $\alpha$ is a function-free FOL formula. The semantics are defined w.r.t.\ the groundings of the FOL formulas, relative to some finite set of constants $\Delta$, called the domain. An MLN $\Phi$ induces the probability distribution on possible worlds $\omega \in \Omega$ over a given domain:
\begin{equation}\label{eq:mln}
    P_{\Phi}(\omega) = \frac{1}{Z} \exp \left(\sum_{(\alpha,w) \in \Phi} w \cdot N(\alpha,\omega)\right),
\end{equation}
where $N(\alpha, \omega)$ is the number of groundings of $\alpha$ satisfied in $\omega$ (when $\alpha$ does not contain any variables, we define $N(\alpha,\omega) = \mathds{1}(\omega \models \alpha)$), and $Z$, called {\em partition function}, is a normalization constant to ensure that $p_{\Phi}$ is a probability distribution. We also allow infinite weights. A weighted formula of the form $(\alpha,+\infty)$ is understood as a hard constraint imposing that all worlds $\omega$ in which $N(\alpha,\omega)$ is not maximal have zero probability (this can also be deduced by taking the limit $w \rightarrow +\infty$). If all formulas in an MLN have at most $k$ variables, we call such an MLN {\em $k$-variable}.

\subsection{Weighted First-Order Model Counting}\label{sec:wfomc}

Computation of the partition function $Z$ can be converted to {\em first-order weighted model counting (WFOMC)}.

\begin{definition}[WFOMC, \citeauthor{van2011lifted}, \citeyear{van2011lifted}]
Let $\Omega$ be a set of possible worlds (often $\Omega$ is the set of all possible worlds over some given domain $\Delta$), $w(P)$ and $\overline{w}(P)$ be functions from predicates to complex\footnote{Normally, in the literature, the weights of predicates are real numbers. However, we will need complex-valued weights in this paper, therefore we define the WFOMC problem accordingly using complex-valued weights.} numbers (we call $w$ and $\overline{w}$ {\em weight functions}) and let $\Gamma$ be an FOL sentence. Then $\operatorname{WFOMC}(\Gamma,w,\overline{w},\Omega) =$
$$
     = \sum_{\omega \in \Omega : \omega \models \Gamma} \prod_{a \in \mathcal{P}(\omega)} w(\textit{Pred}(a)) \prod_{a \in \mathcal{N}(\omega)} \overline{w}(\textit{Pred}(a))
$$
where $\mathcal{P}(\omega)$ and $\mathcal{N}(\omega)$ denote the positive literals that are true and false in $\omega$, respectively, and $\textit{Pred}(a)$ denotes the predicate of $a$ (e.g. $\textit{Pred}(\textit{friends}(\textit{Alice},\textit{Bob})) = \textit{friends}$).
\end{definition}

To compute the partition function $Z$ using weighted model counting, we proceed as \citep{van2011lifted}. Let an MLN $\Phi = \{(\alpha_1,w_1),\dots,(\alpha_m,w_m) \}$ over a set of possible worlds $\Omega$ be given.
For every $(\alpha_j,w_j) \in \Phi$, where the free variables in $\alpha_j$ are exactly $x_1$, $\dots$, $x_k$ and where $w \neq +\infty$, we create a new formula
$
    \forall x_1,\dots,x_k : \xi_j(x_1,\dots,x_k) \Leftrightarrow \alpha_j(x_1,\dots,x_k)
$
where $\xi_j$ is a new fresh predicate. When $w = +\infty$, we instead create a new formula $\forall x_1,\dots,x_k : \alpha_j(x_1,\dots,x_k)$. We denote the resulting set of new formulas $\Gamma$. Then we set
$w(\xi_j) = \exp{\left(w_j \right)}$
and $\overline{w}(\xi_j) = 1$ and for all other predicates we set both $w$ and $\overline{w}$ equal to 1. It is easy to check that then $\mathbf{WFOMC}(\Gamma,w,\overline{w},\Omega) = Z$, which is what we needed to compute. To compute the marginal probability of a given FOL sentence $\gamma$, we have $\textit{P}_{\Phi}[X \models q] = \frac{\mathbf{WFOMC}(\Gamma \cup \{ q \}, w, \overline{w},\Omega)}{\mathbf{WFOMC}(\Gamma, w, \overline{w},\Omega)}$ where $X$ is sampled from the MLN.

\subsection{Domain-Lifted Inference}

Importantly, there are classes of FOL sentences for which weighted model counting is polynomial-time. In particular, let $\Omega$ be the set of all possible worlds over a given domain $\Delta$ and a given set of relations $\mathcal{R}$. As shown in \citep{van2014skolemization}, when the theory $\Gamma$ consists only of FOL sentences, each of which contains at most two logic variables, the weighted model count can be computed in time polynomial in the number of elements in the domain $\Delta$. It follows from the translation described in the previous section that this also means that computing the partition function of $2$-variable MLNs can be done in time polynomial in the size of the domain. This is not the case in general when the number of variables in the formulas is greater than two unless P = \#P$_1$~\citep{beame2015symmetric}.\footnote{\#P$_1$ is the set of \#P problems over a unary alphabet.} Within statistical relational learning, the term used for problems that have such polynomial-time algorithms is {\em domain liftability}.

\begin{definition}[Domain liftability]
An algorithm for computing WFOMC with real weights is said to be domain-liftable if it runs in time polynomial in the size of the domain.
\end{definition}



One can show, by inspecting the respective domain-lifted algorithms from the literature, e.g.\ \citep{van2011lifted,van2014skolemization,beame2015symmetric} that these algorithms can be modified to yield domain-lifted algorithms for WFOMC with complex weights (we refer to \citeauthor{kuzelka.complex}, \citeyear{kuzelka.complex} for details).

\section{Count Distribution and Its DFT}

In this section we will deal with {\em count distributions} induced by MLNs, which are distributions of the numbers of true groundings of given formulas. Let $\Phi = \{(\alpha_1,w_1), \dots, (\alpha_m, w_m) \}$ be an MLN, $\Psi = \{\beta_1,\dots,\beta_{m'} \}$ be a set of FOL formulas and a domain $\Delta$. We first define a notation for the vectors of the ``count-statistics'' on a given possible world $\omega$:
$$\mathbf{N}(\Psi,\omega) \stackrel{def}{=} [N(\beta_1,\omega),\dots,N(\beta_m,\omega)].$$

\noindent Now we can define count distributions formally.


\begin{definition}[Count Distribution]
Let $\Phi = \{(\alpha_1,w_1), \dots, (\alpha_m,w_m) \}$ be an MLN defining a distribution over a set of possible worlds $\Omega$ and $\Psi = \{\beta_1,\dots,\beta_{m'} \}$ be a set of FOL formulas. The count distribution of $\Phi$ is the distribution of the $m'$-dimensional vectors of non-negative integers $\mathbf{n}$ given by
$$q_{\Psi,\Phi}(\mathbf{n}) =  \sum_{\omega \in \Omega : \mathbf{N}(\Psi,\omega) = \mathbf{n}} p_{\Phi}(\omega)$$
where $p_{\Phi}$ is the distribution given by the MLN $\Phi$.
\end{definition}


\subsection{Computing Count Distributions}

At first it is not obvious how to compute a count distribution of a given MLN. Here we extend an approach based on discrete Fourier transform which we introduced in \citep{kuzelka.complex}. Previously in \citep{kuzelka.complex}, we only assumed the case where $\Psi$ and $\Phi$ contained the same set of formulas. We lift this restriction here.

Let $\Omega$ be the set of all possible worlds on a given domain $\Delta$ and a given set of relations $\mathcal{R}$.
We want to compute the DFT of $q_{\Psi,\Phi}(\mathbf{n})$ which is a real-valued function of $m$-dimensional integer vectors. We can restrict the domain\footnote{Here, {\em domain} refers to the domain of a mathematical function, not to a {\em domain} as a set of domain elements.} of $q_{\Psi,\Phi}(\mathbf{n})$ to the set $\mathcal{D} = \left\{0,1,\dots, |\Delta|^{|\textit{vars}(\beta_1)|} \right\} \times \left\{0,1,\dots, |\Delta|^{|\textit{vars}(\beta_2)|}\right\} \times \left\{ 0, 1, \dots, |\Delta|^{|\textit{vars}(\beta_m)|} \right\}$. 

From the definition of DFT we then have
\begin{equation}\label{dftq1}
    g_{\Psi,\Phi}(\mathbf{k}) = \mathcal{F} \left\{ q_{\Psi,\Phi} \right\} = \sum_{\mathbf{n} \in \mathcal{D}} q_{\Psi,\Phi}(\mathbf{n})  e^{ - i 2 \pi \langle \mathbf{k}, \mathbf{n} / \mathbf{M} \rangle }
\end{equation}
where $\mathbf{M} = \left[|\Delta|^{|\textit{vars}(\beta_1)|}+1,\dots,|\Delta|^{|\textit{vars}(\beta_m)|}+1\right]$ and the division in $\mathbf{n} / \mathbf{M}$ is again component-wise.

For notational convenience, we define $\mathbf{w} = [w_1,w_2,\dots,w_m]$ to be the vector of weights of the formulas from the MLN $\Phi$. Plugging in the definition of $q_{\Phi}(\mathbf{n})$ into (\ref{dftq1}), we obtain
\begin{multline*}
    g_{\Psi,\Phi}(\mathbf{k})= \sum_{\mathbf{n} \in \mathcal{D}}\sum_{\omega \in \Omega : \mathbf{N}(\Psi,\omega) = \mathbf{n}} p_{\Phi}(\omega)  e^{ - i 2 \pi \langle \mathbf{k}, \mathbf{n} / \mathbf{M} \rangle } \\
    = \sum_{\mathbf{n} \in \mathcal{D}}\sum_{\omega \in \Omega : \mathbf{N}(\Psi,\omega) = \mathbf{n}} \frac{1}{Z} e^{\langle \mathbf{w}, \mathbf{N}(\Phi,\omega) \rangle}  e^{ - i 2 \pi \langle \mathbf{k}/\mathbf{M}, \mathbf{n} \rangle } \\
    = \sum_{\mathbf{n} \in \mathcal{D}}\sum_{\omega \in \Omega : \mathbf{N}(\Psi,\omega) = \mathbf{n}} \frac{1}{Z} e^{\langle \mathbf{w}, \mathbf{N}(\Phi,\omega) \rangle}  e^{ - i 2 \pi \langle \mathbf{k}/\mathbf{M}, \mathbf{N}(\Psi,\omega) \rangle } \\
    = \frac{1}{Z} \sum_{\omega \in \Omega}  e^{\langle \mathbf{w}, \mathbf{N}(\Phi,\omega) \rangle - i 2 \pi \langle \mathbf{k}/\mathbf{M}, \mathbf{N}(\Psi,\omega) \rangle }.
\end{multline*}
Now the last expression is already something that can be computed using WFOMC over complex numbers. First, $Z$ is the partition function of the MLN $\Phi$, which can be computed using WFOMC as described in Section \ref{sec:wfomc}.
The sum $\sum_{\omega \in \Omega}  e^{\langle \mathbf{w}, \mathbf{N}(\Phi,\omega) \rangle - i 2 \pi \langle \mathbf{k}/\mathbf{M}, \mathbf{N}(\Psi,\omega) \rangle }$ can be computed in a completely analogical way. For every $(\alpha_j,w_j) \in \Phi$, where the free variables in $\alpha_j$ are exactly $x_1$, $\dots$, $x_k$ and where $w \neq +\infty$, we create a new formula
$
    \forall x_1,\dots,x_k : \xi_j(x_1,\dots,x_k) \Leftrightarrow \alpha_j(x_1,\dots,x_k)
$
where $\xi_j$ is a new fresh predicate. When $w = +\infty$, we instead create a new formula $\forall x_1,\dots,x_k : \alpha_j(x_1,\dots,x_k)$. Similarly, for every $\beta_j \in \Psi$, where the free variables in $\beta_j$ are exactly $x_1$, $\dots$, $x_k$, we create a new formula
$
    \forall x_1,\dots,x_k : \xi_{\beta_j}(x_1,\dots,x_k) \Leftrightarrow \beta_j(x_1,\dots,x_k).
$
Then we set
$w(\xi_{\alpha_j}) = \exp{\left(w_j \right)}$ and $\overline{w}(\xi_{\alpha_j}) = 1$ for all $(\alpha_j,w_j) \in \Phi$,
$w(\xi_{\beta_j}) = \exp{\left(-i 2 \pi [\mathbf{k}]_j / M_j\right)}$ and $\overline{w}(\xi_{\beta_j}) = 1$,
and for all other predicates we set both $w$ and $\overline{w}$ to 1.

Thus, we can compute the DFT of a count distribution using a polynomial number (in $|\Delta|$) of queries to a WFOMC oracle. Importantly, we do not need to add explicit cardinality constraints to the MLN or modify the formulas in it or in the set $\Psi$ in any way as long as the WFOMC oracle works with complex weights. The next proposition follows from what we showed above.\footnote{Proposition \ref{prop:dft} could be made a bit stronger since there are classes of WFOMC problems and MLNs beyond the 2-variable fragment that are domain liftable. However, we prefer to present the simpler version here as it is easier to understand.}

\begin{proposition}\label{prop:dft}
Let $\Omega$ be the set of all possible worlds on a given domain $\Delta$ and a given set of relations $\mathcal{R}$. Let $\Phi = \{(\alpha_1,w_1),\dots,(\alpha_m,w_m)\}$ be an MLN and $\Psi = \{\beta_1,\dots,\beta_{m'}\}$ be a set of FOL formulas.
If all the formulas $\alpha_1$, $\dots$, $\alpha_m$ and $\beta_1$, $\dots$, $\beta_{m'}$ contain at most 2 variables then the DFT of the count distribution $q_{\Psi,\Phi}(\mathbf{n})$ can be computed in time polynomial in the domain size $|\Delta|$.
\end{proposition}
\begin{proof}
The proof follows from the discussion above.
\end{proof}

Now, we know how to compute DFT of count distributions but we have not yet explained how to compute the count distributions themselves. That is actually very easy. We can just take the DFT and invert it. Thus, we obtain the next corollary.

\begin{corollary}
Let $\Omega$, $\Delta$, $\Phi$ and $\Psi$ be as in Proposition \ref{prop:dft}. Then the count distribution $q_{\Psi,\Phi}(\mathbf{n})$ can be computed in time polynomial in the domain size $|\Delta|$.
\end{corollary}

\section{MLNs with Cardinality Constraints}

In this paper, a {\em cardinality constraint} $(\Psi,g)$ is a pair consisting of a set of formulas $\Psi = \{\beta_1,\dots,\beta_d \}$ and a function $g : \mathbb{N}^d \rightarrow \{0,1 \}$.
A distribution $p$ satisfies a given cardinality constraint $(\Psi,g)$ if $p(\omega) = 0$ for all $\omega$ s.t.    $g(\mathbf{N}(\Psi,\omega)) = 0$.
We can use cardinality constraints to define MLN-like distributions:
$$p_{\Phi,(\Psi,g)}(\omega) = \frac{g(\mathbf{N}(\Psi,\omega))}{Z} \exp\left(\sum_{(\alpha,w) \in \Phi} w \cdot N(\alpha,\omega)\right)$$
where $$Z = \sum_{\omega \in \Omega} g(\mathbf{N}(\Psi,\omega)) \cdot \exp\left(\sum_{(\alpha,w) \in \Phi} w \cdot N(\alpha,\omega)\right)$$
is a normalization constant. Let $p_{\Phi}$ be a distribution given by the MLN $\Phi$ and $p_{\Phi,(\Psi,g)}$ be a distribution given by the same MLN with the cardinality constraint $(\Psi,g)$. Then for all $\omega_1$, $\omega_2$ such that $f(\omega_1) = f(\omega_2) = 1$, it holds that
$\frac{p_{\Phi}(\omega_1)}{p_{\Phi}(\omega_2)} = \frac{p_{\Phi,(\Psi,g)}(\omega_1)}{p_{\Phi,(\Psi,g)}(\omega_2)}.$
This means that adding cardinality constraints in this way does not affect ratios of probabilities of those possible worlds which satisfy the constraints. This will be useful in the next section.

\begin{example}
For instance, if we have an MLN $\Phi = \{ (\textit{sm}(x) \wedge \textit{fr}(x,y) \Rightarrow \textit{sm}(y), w) \}$ modelling how smoking behaviour of people affects smoking habits of their friends, we can use cardinality constraints to express that exactly $M$ people are smokers. For this, we can set $\Psi = \{\textit{sm}(x) \}$ and $g(n) = \mathds{1}(n = M)$.
\end{example}

Inference in MLNs with cardinality constraints can be done using inference over count distributions which we already know how to do from the previous section. Let us have an MLN $\Phi$ with a cardinality constraint $(\Psi,g)$. Suppose that we want to compute the probability of a marginal query $P[X \models \gamma]$ for an FOL sentence $\gamma$. We construct the count distribution $q_{\Phi,\Psi \cup \{ \gamma \}}$ as described in the previous section and compute
\begin{equation*}
P[X \models \gamma] = \sum_{\mathbf{j} \in \mathcal{D} : [\mathbf{j}]_{m+1} = 1} g(\mathbf{j}) \cdot q_{\Phi,\Psi \cup \{ \gamma \}}(\mathbf{j})
\end{equation*}
where $\mathcal{D} = \{0,1,\dots, M_1 \} \times \{0,1,\dots, M_2\} \times \{ 0, 1, \dots, M_m \} \times \{ 0, 1 \}$ where $M_1 = |\Delta|^{|\textit{vars}(\beta_1)|}$, $M_2 = |\Delta|^{|\textit{vars}(\beta_2)|}$, \dots, $M_m = |\Delta|^{|\textit{vars}(\beta_m)|}$. After that we are done. Notice that the condition $[\mathbf{j}]_{m+1} = 1$ in the sum makes sure that we are only summing up probabilities of possible worlds in which $\gamma$ is true. It follows from the discussion in this and the previous section that the next proposition holds.

\begin{proposition}\label{prop:cardinality}
Let $\Omega$ be the set of all possible worlds on a given domain $\Delta$ and a given set of relations $\mathcal{R}$. Let $\Phi = \{(\alpha_1,w_1),\dots,(\alpha_m,w_m)\}$ be a $2$-variable MLN over $\Omega$. Let $(\Psi,g)$ be a cardinality constraint where each $\beta \in \Psi$ has at most two variables and let $\gamma$ be an FOL sentence with at most 2 logic variables.
Then the probability of the marginal query $P[X \models \gamma]$, where $X$ is sampled from the distribution given by $\Phi$ with the cardinality constraint $(\Psi,g)$, can be computed in time polynomial in the domain size~$|\Delta|$.
\end{proposition}
\begin{proof}
The proof follows from Proposition \ref{prop:dft} and the discussion above.
\end{proof}

\section{MLNs with Function Constraints}

A {\em function constraint} $\textit{Func}(R_i)$, where $R$ is a relation, is a constraint equivalent to the first order-logic sentence $\forall x \exists_{=1} R_i(x,y)$ which asserts that for every $x$ there is exactly one $y$ such that $R(x,y)$ is true. In this section we show how to extend 2-variable MLNs to handle an arbitrary number of function constraints while still guaranteeing inference in time polynomial in the domain size $|\Delta|$.

We start with the following simple lemma that will allow us to reduce inference in 2-variable MLNs with function (and possibly also cardinality) constraints to inference in 2-variable MLNs with only cardinality constraints.

\begin{lemma}\label{lemma:lemma1}
Let $\Omega$ be the set of all possible worlds on a domain $\Delta$. Let $\Phi$ be a first order logic sentence. Let
$\Psi = \textit{Func}(R_{i_1}) \wedge \dots \wedge \textit{Func}(R_{i_h})$
and
$
    \Psi' = (\forall x \exists y: R_{i_1}(x,y)) \wedge (|R_{i_1}| = |\Delta|) \wedge
    \dots \wedge (\forall x \exists y: R_{i_h}(x,y)) \wedge (|R_{i_h}| = |\Delta|).
$
Then for all $\omega \in \Omega$: $(\omega \models \Phi \wedge \Psi) \Leftrightarrow (\omega \models \Phi \wedge \Psi')$.
\end{lemma}
\begin{proof}
It suffices to show validity of the statement for just one constraint on a relation $R$ (the general case follows easily). 
The constraint $\textit{Func}(R)$ can be rewritten as: (i) $\forall x \exists y : R(x,y)$ and (ii) $\forall x,y,z: R(x,y) \wedge R(x,z) \Rightarrow y = z$.
($\Rightarrow$) It follows from (i) that $|R| \geq |\Delta|$. If $|R| > |\Delta|$ then by the pigeon-hole principle, there must be at least one $C \in \Delta$ such that $R(C,D)$ and $R(C,D')$ for some $D \neq D' \in \Delta$ which contradicts (ii). Hence, $\textit{Func}(R)$ implies $|R| = |\Delta|$ and $\forall x\exists y : R(x,y)$.
($\Leftarrow$) What we need to show is that if $(\forall x \exists y: R(x,y)) \wedge (|R| = |\Delta|)$ holds then (i) and (ii) must hold as well. Clearly, (i) must hold. So let us suppose, for contradiction, that $(\forall x \exists y: R(x,y)) \wedge (|R| = |\Delta|)$ holds but there is some $C \in \Delta$ such that $R(C,D)$ and $R(C,D')$ for some $D \neq D' \in \Delta$. We have $|\{ (x,y) \in \Delta^2 | R(x,y) \wedge x \neq C \}| \geq |\Delta|-1$ (from $\forall x \exists y : R(x,y)$). Therefore it is easy to see that $|R| \geq |\{ (x,y) \in \Delta^2 | R(x,y) \wedge x \neq C \}| + 2 > |\Delta|$, which is a contradiction.
\end{proof}

\noindent Note that the constraints $|R_{i_1}| = |\Delta|$, $\dots$, $|R_{i_h}| = |\Delta|$ can easily be represented as cardinality constraints.

\begin{proposition}\label{prop:function}
Let $\Omega$, $\Delta$, $\mathcal{R}$, $\gamma$ and $\Phi$ be as in Proposition \ref{prop:cardinality}. Let $\Psi$ be a conjunction of functional constraints.
Then the probability of the marginal query $P[X \models \gamma]$, where $X$ is sampled from the distribution given by $\Phi$ with the function constraints $\Psi$, can be computed in time polynomial in the domain size~$|\Delta|$.
\end{proposition}
\begin{proof}
The proof follows from Lemma \ref{lemma:lemma1}, Proposition \ref{prop:cardinality} and the discussion above.
\end{proof}


Next we illustrate the methods presented in this paper on an example.

\begin{figure}
\centering
\includegraphics[width=0.49\linewidth]{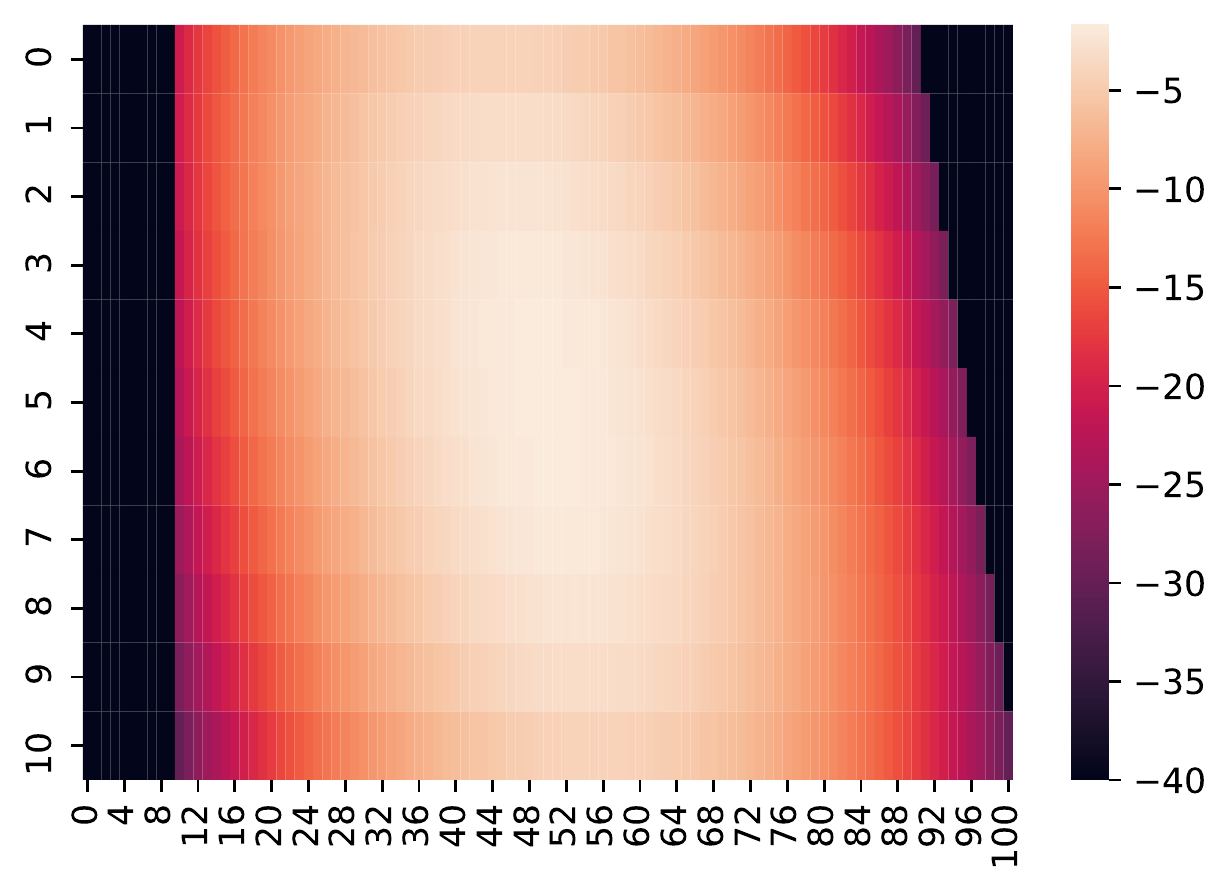}
\includegraphics[width=0.49\linewidth]{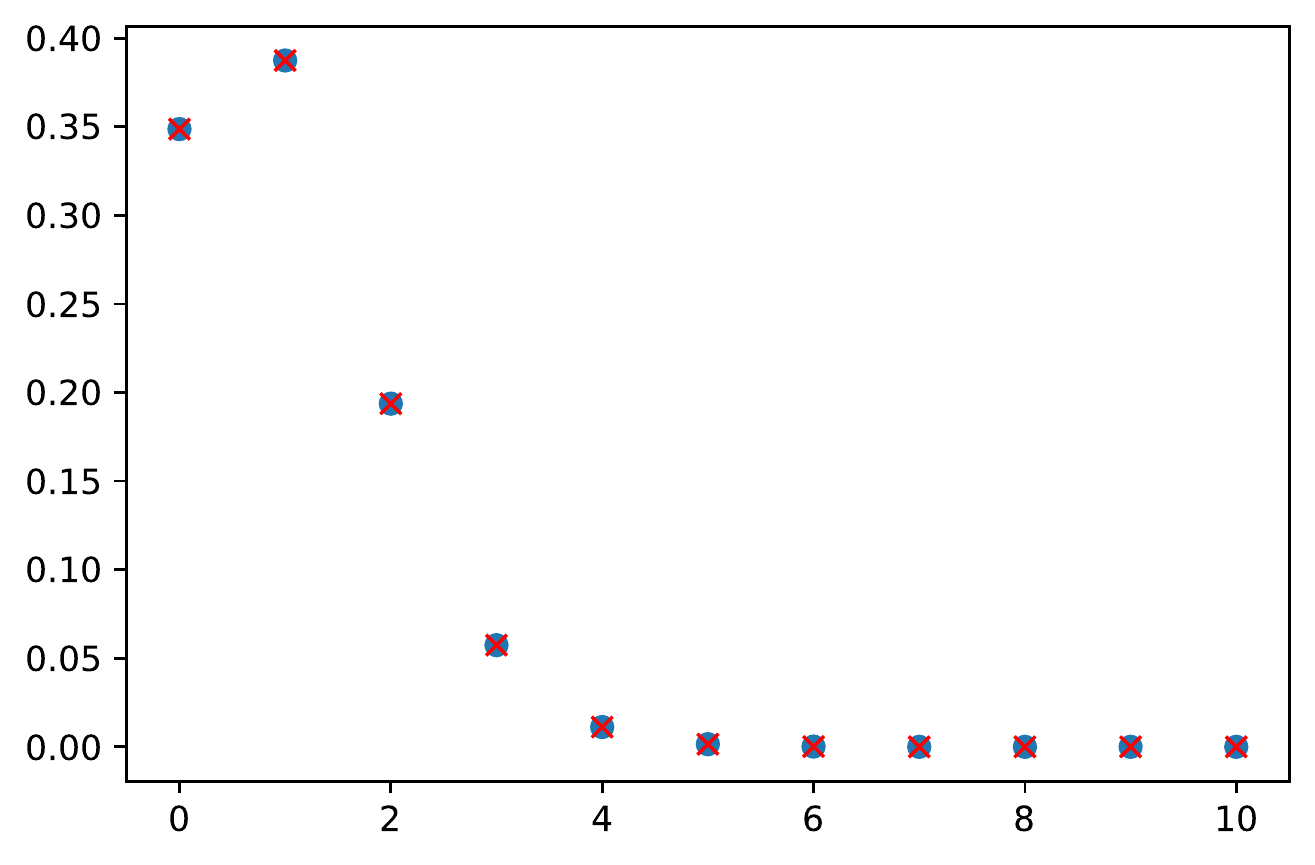}
\caption{{\bf Left:} The count distribution from Example \ref{example}. {\bf Right: } The distribution of the number of fixed points (see Example \ref{example}).}\label{figure}
\end{figure}

\begin{example}\label{example}
How many fixed points does a uniformly sampled function from $\{1,2,\dots,n \}$ to itself have? We can answer this question using MLNs with functional constraints. First, we define $\Phi = \{ (\forall x \exists y : f(x,y), +\infty) \}$. Since we need to enforce the cardinality constraint $|f| = |\Delta|$ (cf discussion in the previous section about encoding functional constraints), we will need the formula $\beta_1 = f(x,y)$. Since we are interested in the number of fixed points, we will also need the formula $\beta_2 = f(x,x)$. Next we define $\Psi = \{\beta_1,\beta_2\}$. Then, using DFT and WFOMC, we compute the count distribution $q_{\Phi,\Psi}$, which is shown in the left panel of Figure \ref{figure} for $n = 10$. Note that the MLN $\Phi$ itself does not model distribution over functions but only over relations $f(x,y)$ which must satisfy $\forall x \exists y : f(x,y)$ but which may or may not be functions. However, we can extract the distribution that we wanted to compute from the count distribution of this MLN. In particular the probability that a uniformly sampled function has $k$ fixed points is equal to $q_{\Phi,\Psi}(|\Delta|,k)/Z'$ where $Z' = \sum_{j = 1}^{|\Delta|} q_{\Phi,\Psi}(|\Delta|,j)$. We show the computed distribution in the right panel of Figure \ref{figure} (blue circles). As a sanity check, we also computed the distribution analytically using the formula $\binom{n}{k} (n-1) ^{n-k}/n^n$ and displayed it in the same plot (red crosses). As expected, the values computed using the two approaches are the same.
\end{example}

\section{Conclusions}

In this paper we have shown how WFOMC with complex weights can be used to obtain new domain-liftability results in a rather straightforward and, arguably, elegant way. We hope that the general approach presented here can lead to further new domain liftability results. There are many things that can still be done from here. First it is possible to get rid of the complex numbers, at the cost of slightly more complicated analysis, either using the {\em number-theoretic transform} in place of complex DFT or using polynomial interpolation. Second, as pointed out in \citep{DBLP:conf/lics/KuusistoL18}, domain liftability for FO$_{2}$ with an arbitrary number of function constraints implies domain liftability for the two-variable logic with counting. So our results should also be relevant there.

\bibliographystyle{named}
\bibliography{kr}

\end{document}